\documentclass{article} 
\usepackage{iclr2027_conference,times}

\usepackage{amsmath,amsfonts,bm}

\def\eqref#1{equation~\ref{#1}}
\def\Eqref#1{Equation~\ref{#1}}

\def\1{\bm{1}}

\def\ra{{\textnormal{a}}}

\def\rx{{\textnormal{x}}}

\def\rva{{\mathbf{a}}}

\def\erva{{\textnormal{a}}}

\def\ervx{{\textnormal{x}}}

\def\rmA{{\mathbf{A}}}

\def\vmu{{\bm{\mu}}}
\def\vtheta{{\bm{\theta}}}
\def\va{{\bm{a}}}

\def\ve{{\bm{e}}}

\def\vx{{\bm{x}}}

\def\eva{{a}}

\def\mA{{\bm{A}}}

\def\mH{{\bm{H}}}
\def\mI{{\bm{I}}}
\def\mJ{{\bm{J}}}

\def\mX{{\bm{X}}}

\def\mSigma{{\bm{\Sigma}}}

\DeclareMathAlphabet{\mathsfit}{\encodingdefault}{\sfdefault}{m}{sl}
\SetMathAlphabet{\mathsfit}{bold}{\encodingdefault}{\sfdefault}{bx}{n}
\newcommand{\tens}[1]{\bm{\mathsfit{#1}}}
\def\tA{{\tens{A}}}

\def\tX{{\tens{X}}}

\def\gG{{\mathcal{G}}}

\def\sA{{\mathbb{A}}}
\def\sB{{\mathbb{B}}}

\def\sS{{\mathbb{S}}}

\def\emA{{A}}

\newcommand{\etens}[1]{\mathsfit{#1}}

\def\etA{{\etens{A}}}

\newcommand{\E}{\mathbb{E}}

\newcommand{\R}{\mathbb{R}}

\newcommand{\KL}{D_{\mathrm{KL}}}
\newcommand{\Var}{\mathrm{Var}}

\newcommand{\Cov}{\mathrm{Cov}}

\newcommand{\normltwo}{L^2}
\newcommand{\normlp}{L^p}

\newcommand{\parents}{Pa} 

\DeclareMathOperator{\graph}{graph}

\usepackage{hyperref}
\usepackage{url}
\usepackage{graphicx}
\usepackage{capt-of}
\usepackage{booktabs}
\usepackage{amsthm}
\usepackage{mathtools}
\usepackage{algorithm}
\usepackage{comment}
\usepackage{booktabs,tabularx,xcolor}
\usepackage[noend]{algpseudocode}
\theoremstyle{plain}
\newtheorem{theorem}{Theorem}[section]
\newtheorem{lemma}[theorem]{Lemma}
\newtheorem{proposition}[theorem]{Proposition}
\newtheorem{rem}[theorem]{Remark}
\theoremstyle{definition}
\newtheorem{definition}[theorem]{Definition}

\title{Lagrangian--Hamiltonian Flows for Video Prediction and Image Generation: A Symplectic Perspective}

\author{
Jiawei Hu\\
Department of Mathematics and Statistics, Boston University\\
\href{jiaweihu@bu.edu}{\texttt{jiaweihu@bu.edu}}
}

\iclrfinalcopy
\begin{document}

\maketitle
\pagestyle{plain}

\begin{abstract}
    We introduce LHFM, a geometric framework for learning image dynamics.
    Drawing on structures central to classical mechanics, symplectic
    geometry, and geometric quantization, LHFM represents each image as an
    exact Lagrangian graph and models its evolution through image-dependent
    Hamiltonian flows, which yield a transport--source parameterization of
    image velocities. Our primary application is deterministic video
    prediction: LHFM-V is a recurrent model that advances frames by
    integrating predicted transport and source fields, and achieves the
    lowest reported FLOP count among the compared recurrent models with
    similar prediction accuracy. The image variant, LHFM-I, shows that the
    same construction is compatible with flow matching: in a matched
    experiment, it attains a lower FID than the flow-matching baseline.
    Code is available at
    \url{https://github.com/Lilas-Q/Geometric_quantization_models}.
    \end{abstract}
\section{Introduction}
Diffusion models and flow-based generative models generate images by
transforming a simple distribution into the data distribution
\citep{ho2020denoising,song2021scorebased,lipman2023flow}.
Flow matching (FM) provides a simulation-free training objective for continuous
normalizing flows by regressing a time-dependent vector field onto
conditional vector fields \citep{lipman2023flow}.
In the original FM formulation, the network directly models a vector field
on the data space where each image represents a point of this space.

A similar perspective applies to video prediction: conditioned on the
observed sequence, a learned flow evolves the observed frames into
future frames.
Thus, image generation and video prediction can both be formulated as
learning a flow on image space, although their initial conditions and
training objectives differ.

Image dynamics involve both spatial structure (pixel
locations) and color (channel values). In standard FM, however,
the vector field acts on channel values at fixed pixel locations, so changes
in spatial structure are reflected only implicitly through changes in color.
This motivates a question: how can these two aspects of image evolution be
expressed through a common geometric representation?

We answer this question by introducing \emph{Lagrangian--Hamiltonian
Flow Matching} (LHFM), a geometric framework that combines Lagrangian
image representations with Hamiltonian dynamics. A Lagrangian submanifold can be viewed as a ``generalized point" which encodes the space and phase simultaneously. Motivated by this, we  represent each image as an exact Lagrangian graph that jointly encodes pixel locations and channel values.
Image evolution is then described by Hamiltonian flows that transport
these graphs, as illustrated in Figure~\ref{fig:lhfm-hamiltonian-flow}.
Throughout, ``Lagrangian'' refers to a submanifold rather than to an action functional.

\begin{figure}[!ht]
    \centering
    \begin{minipage}[c]{0.63\linewidth}
        \centering
        \includegraphics[width=\linewidth]{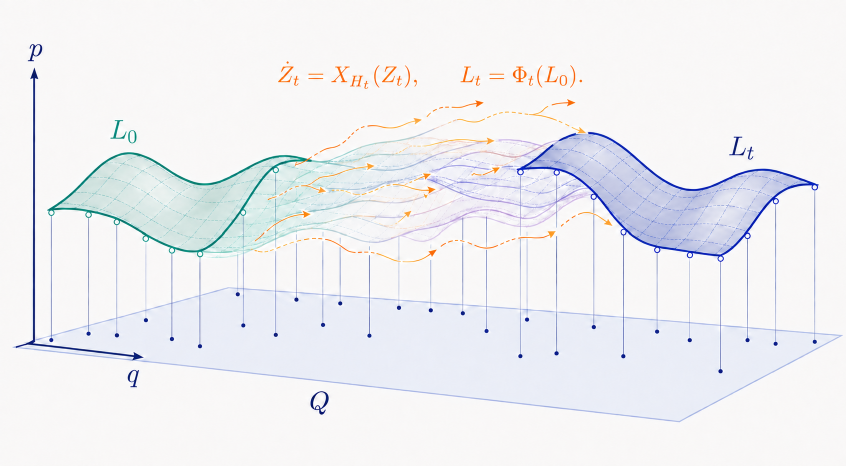}
    \end{minipage}\hfill
    \begin{minipage}[c]{0.33\linewidth}
        \caption{\textbf{Hamiltonian deformation of a Lagrangian submanifold.}
        For a fixed image trajectory, the time-dependent ambient Hamiltonian
        $H_t$ transports $L_0$ to $L_t=\Phi_t(L_0)$.
         The surfaces schematically represent
        higher-dimensional exact Lagrangians.}
        \label{fig:lhfm-hamiltonian-flow}
    \end{minipage}
\end{figure}

A central issue is compatibility: a Hamiltonian flow preserves the Lagrangian property but need not preserve a graph representation. We show that a Hamiltonian transports a prescribed family of
exact Lagrangian graphs if and only if it satisfies a Hamilton--Jacobi
condition.  This characterization leads to a natural parameterization of the vector field on the data space in terms of a \emph{transport vector field}, which moves spatial locations, and a \emph{source field}, which changes channel values along that motion. The construction is detailed in Section~\ref{sec:image-lagrangian-fm}.

We instantiate LHFM as LHFM-I for image generation and LHFM-V for
deterministic video prediction, where I and V denote image and video,
respectively. We evaluate the two variants on unconditional CIFAR-10
generation and Moving MNIST prediction.
In a matched comparison, LHFM-I improves upon the standard
independent conditional flow-matching baseline. LHFM-V uses a recurrent architecture and has the lowest reported FLOP
count among the compared recurrent models with similar prediction
accuracy.

The primary focus of this work is deterministic video prediction:
LHFM-V is our main method, and the video experiments constitute the
main empirical evaluation. LHFM-I is included to show that the same
transport--source parameterization is compatible with flow-matching
image generation without changing the training objective. We therefore
evaluate LHFM-I only in a controlled comparison against a matched
I-CFM baseline, and we do not claim state-of-the-art image generation.

Our contributions are:
\begin{itemize}
    \item to our knowledge, the first representation of images as exact
    Lagrangian graphs, with a Hamilton--Jacobi characterization (classical in
    form; proved in Appendix~\ref{app:image-graph-hamiltonians-proof}) of the
    Hamiltonians that transport them;
    \item a transport--source parameterization derived from it, which contains
    the standard flow-matching velocity as the case $U_t\equiv 0$; and
    \item empirical studies on deterministic Moving MNIST prediction,
    where LHFM-V has the lowest FLOP count among the compared recurrent
    video predictors (Table~\ref{tab:video-prediction-results}), and on
    CIFAR-10 generation, where LHFM-I attains a lower FID than a matched
    I-CFM baseline, demonstrating compatibility with flow matching.
\end{itemize}

\section{Related Work}
\label{sec:related-work}

\paragraph{Flow matching.}
Flow matching trains continuous normalizing flows through
simulation-free regression of conditional vector fields
\citep{lipman2023flow}.
Conditional flow matching extends this construction to general
source and target distributions, with independent and optimal
transport couplings as particular choices \citep{tong2024improving}.
LHFM-I retains the independent conditional
flow-matching objective but parameterizes the image velocity
through a transport vector field and a source field.

\paragraph{Geometric formulations.}
Riemannian Flow Matching extends flow matching to manifolds
\citep{chen2024flow}, while Metric Flow Matching uses data-dependent
Riemannian metrics to construct conditional paths
\citep{kapusniak2024metric}.
LHFM takes a different approach: it represents each image as an
exact Lagrangian graph and characterizes Hamiltonian flows that
transport these representations.
The underlying connection between Lagrangian submanifolds and
the Hamilton--Jacobi equation is classical
\citep{carinena2006geometric}, and transport--source image dynamics
also appear in image metamorphosis
\citep{trouve2005metamorphoses,holm2009eulerpoincare}.
Our contribution applies this connection to image representations
and derives a transport--source parameterization of image dynamics.

\paragraph{Video prediction.}
Recurrent video predictors advance a hidden state frame by frame, from
ConvLSTM \citep{shi2015convolutional} and the PredRNN family
\citep{wang2017predrnn,wang2022predrnnv2} to PhyDNet, which constrains
part of its latent dynamics with a learned PDE
\citep{leguen2020disentangling}. LHFM-V belongs to this category:
conditioned on observed frames, it recurrently predicts transport and
source fields and evolves the image through their combined dynamics.
Recurrent-free predictors \citep{gao2022simvp,tang2026predformer} and
generative video models \citep{davtyan2023efficient} are outside the
scope of this work.

\section{Geometric Preliminaries and motivation}
\label{sec:geometric-preliminaries}

A symplectic manifold $(X^{2m},\omega)$ is a smooth manifold equipped
with a closed, nondegenerate two-form. A submanifold $L\subset X$ is
\emph{Lagrangian} if $\dim L=m$ and $\omega|_L=0$.
The cotangent bundle $T^*Q$ carries the canonical one-form $\lambda$
and symplectic form $\omega=-d\lambda$.
For a smooth function $S:Q\to\mathbb R$, the graph
$\operatorname{graph}(dS)$ is an exact Lagrangian submanifold:
the pullback of $\lambda$ to the graph equals $dS$
\citep{cannas2001lectures,mcduff2017introduction}.

A time-dependent Hamiltonian $H_t:T^*Q\to\mathbb R$ determines a
Hamiltonian vector field through $\iota_{X_{H_t}}\omega=d_zH_t$.
Its flow preserves the symplectic form and the exactness of Lagrangian
submanifolds, but not necessarily their graph representations.
Our construction uses exact graphs to represent images and the
Hamilton--Jacobi equation to characterize their evolution.

The cotangent bundle $T^*Q$ provides the natural phase space of
classical mechanics, with its canonical coordinates representing
position and momentum. In semiclassical analysis, an oscillatory
state of the form $a(q)e^{\mathrm{i}S(q)/\hbar}$ is microlocally
associated with the Lagrangian graph $L_S=\graph(d S)$, on which
the phase function determines the relation $p=d S(q)$.
Whereas a classical phase-space point specifies position and
momentum simultaneously, the Heisenberg uncertainty principle
prevents a quantum state from being sharply localized in both
variables \citep{kennard1927zur,robertson1929uncertainty}.
This phase-space viewpoint motivates our use of Lagrangian graphs to organize spatial locations, channel values, and their local spatial sensitivities within a common geometric representation.

\begin{table}[!htbp]
    \centering
    \caption{Geometric motivation for LHFM.
      Hamiltonian evolution refers to
    the ambient cotangent bundle, not the image-space vector field.}
    \label{tab:geometric-motivation}

    \small
    \setlength{\tabcolsep}{5pt}
    \renewcommand{\arraystretch}{1.12}

    \begin{tabularx}{\linewidth}{@{}
        >{\raggedright\arraybackslash}p{0.20\linewidth}
        >{\raggedright\arraybackslash}X
        >{\raggedright\arraybackslash}X
        @{}}
        \toprule[0.6pt]
        & \textbf{Classical mechanics}
        & \textbf{Semiclassical (WKB)} \\
        \midrule[0.3pt]
        Representation
        & Point $(q,p)\in T^*Q$
        & Lagrangian graph with amplitude \\

        \specialrule{1pt}{4pt}{4pt}

        & \textbf{Vanilla  Flow Matching}
        & \textbf{LHFM} \\
        \midrule[0.3pt]
        Representation
        & An image point $J$
        & Graph Lagrangian $L_J=\operatorname{graph}(dS_J)$ \\

        \addlinespace[2pt]
        Dynamics
        & A general learned Flow in image space
        & Hamiltonian flow on $T^*Q$ \\
        \bottomrule[0.6pt]
    \end{tabularx}
\end{table}

\section{Lagrangian--Hamiltonian Flow Matching}
\label{sec:image-lagrangian-fm}

We first define a Lagrangian representation of an image and characterize
the Hamiltonian flows that transport these representations. We then use
the induced image dynamics for conditional flow matching and deterministic
video prediction. The representation and continuous geometric construction
are shared; initial conditions, conditioning, losses, and numerical methods
are specified separately for the two applications.

\subsection{Image Space and Continuous Representation}

Let $M=\mathbb{R}^2$ be the spatial coordinate space and
$V=\mathbb{R}^n$ the space of $n$-channel pixel values.
For the pixel-centered grid $\mathcal{G}_N$ of an $N\times N$
image, write $\mathcal{X}_N=V^{\mathcal{G}_N}$ and $d=nN^2$.
The linear full-band cosine extension associates each image
$I\in\mathcal{X}_N$ with a smooth map
$F_I=\mathcal{F}_N(I):M\to V$ satisfying
$F_I|_{\mathcal{G}_N}=I$.
The detailed construction and its properties are given in
Appendix~\ref{app:dct-interpolation}.

\subsection{Lagrangian Image Representation}
\begingroup
\setlength{\abovedisplayskip}{5pt plus 1pt minus 2pt}
\setlength{\belowdisplayskip}{5pt plus 1pt minus 2pt}
\setlength{\abovedisplayshortskip}{0pt plus 1pt}
\setlength{\belowdisplayshortskip}{3pt plus 1pt minus 1pt}
Introduce an auxiliary dual variable $a\in V^*$.  Throughout, we use
the standard Euclidean identifications of vector spaces and their
duals when writing coordinate transposes.
Set $Q=\mathbb{R}^2\times V^*$, and equip its cotangent bundle
$T^*Q$ with canonical coordinates $(x,a;\xi,\eta)$, where
$\xi$ and $\eta$ are dual to $x$ and $a$, respectively.
The canonical one-form and symplectic form are
\begin{equation}
    \lambda=\xi^\top d x+\eta^\top d a,
    \qquad
    \omega=-d\lambda.
    \label{eq:canonical-forms}
\end{equation}
Using the natural duality pairing between $V^*$ and $V$,
we define the generating function $S_I:Q\to\mathbb{R}$
and the associated LHFM encoder by
\begin{equation}
\begin{split}
    S_I(x,a)
    &=\langle a,F_I(x)\rangle
      =a^\top F_I(x),\\
    E(I)=L_I
    &=\graph(d S_I)=\left\{
        \bigl(x,a;DF_I(x)^\top a,F_I(x)\bigr)
        :(x,a)\in Q
      \right\}.
\end{split}
\label{eq:intensity-lift}
\end{equation}
Here $DF_I(x)$ denotes the spatial Jacobian of $F_I$.
As the graph of an exact one-form, $L_I$ is an exact
Lagrangian submanifold of $T^*Q$.

Exactness and injective image recovery are established in
Proposition~\ref{prop:intensity-lift} in
Appendix~\ref{app:exactness-proof}.

\paragraph{Geometric Intuition of Lagrangian representation}
The auxiliary variable $a\in V^*$ acts as a linear probe of
the channel values, with scalar response $S_I(x,a)$ at
spatial location $x$. For an RGB image, $a=e_R$ selects
the red channel, whereas $a=(e_G-e_R)/\sqrt{2}$ probes
the normalized green-minus-red contrast.

    Differentiating with respect to the spatial and probe variables yields
    \[
    d S_I
    =
    \underbrace{DF_I(x)^\top a}_{\xi}\cdot d x
    +
    \underbrace{F_I(x)}_{\eta}\cdot d a.
    \]
    Thus, $\eta$ records the channel values and determines how the response
    changes when the probe is varied, while $\xi$ describes how the same
    response changes under a spatial displacement. 
    \begin{figure}[!t]
        \normalfont
        \centering
        \includegraphics[width=\linewidth]
        {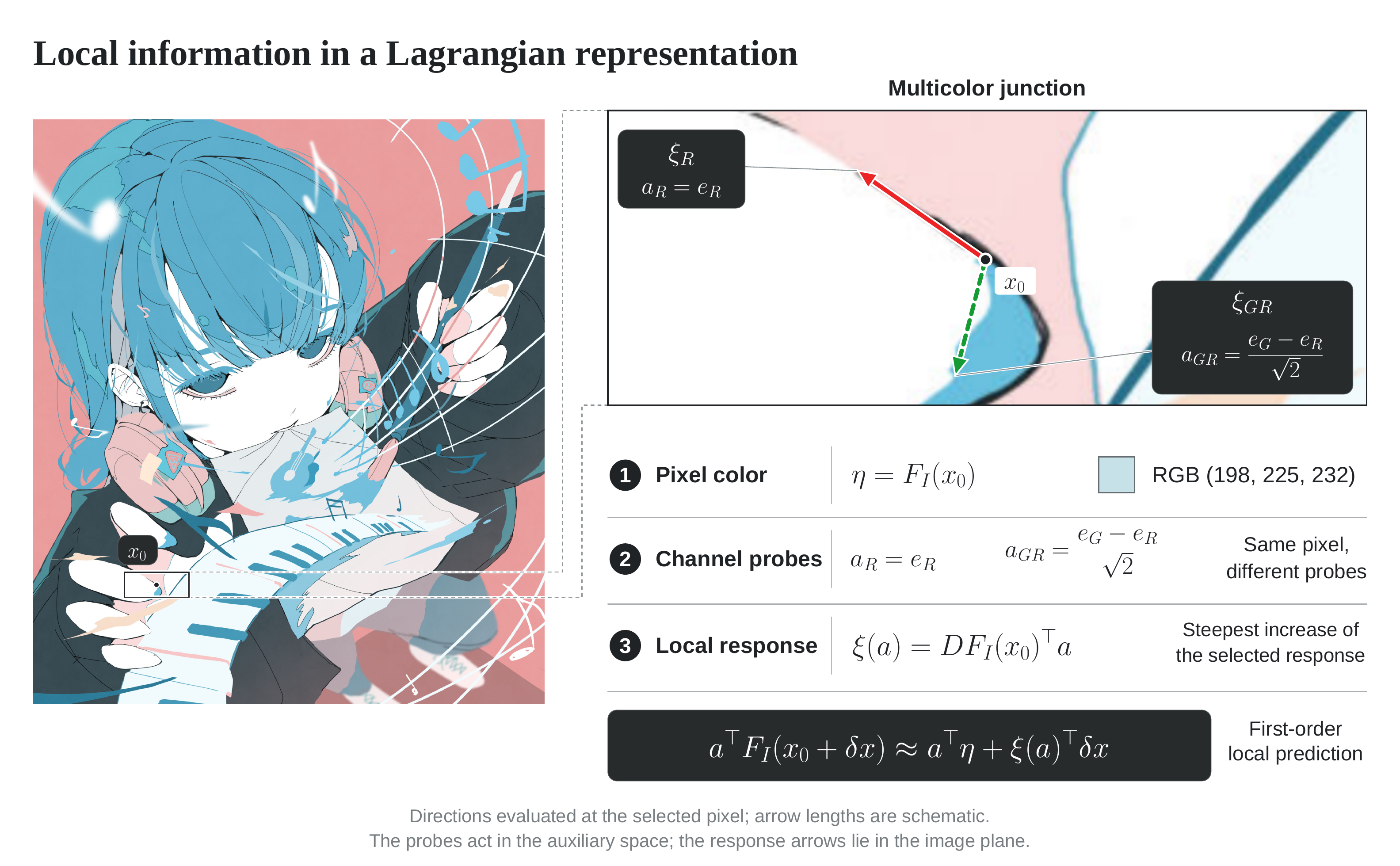}
        \makeatletter
        \let\refstepcounter\H@refstepcounter
        \makeatother
        \caption{\textbf{Local information encoded by the Lagrangian representation.}
        Arrow directions are computed at $x_0$ in Euclidean pixel coordinates; lengths are schematic and not to scale.}
        \label{fig:lagrangian-local-information}
        \end{figure}

As shown in Figure~\ref{fig:lagrangian-local-information},\footnote{Illustration by
\href{https://fromtheasia.com/illustration/2024-2025}{NoCopyrightGirl (NCG)},
reproduced and annotated under the artist's
\href{https://fromtheasia.com/usagerules}{published usage rules}.
The artwork is used solely for illustration and was not used for model training.}
the covector $\xi(a)=DF_I(x_0)^\top a$, gives the direction of steepest
increase of the probe response at $x_0$. The red arrow ($a=e_R$) indicates
the direction in which the red-channel value increases most rapidly,
whereas the green arrow ($a=(e_G-e_R)/\sqrt{2}$) indicates the
direction of steepest increase in the green-minus-red contrast.

The Lagrangian graph therefore organizes
    pixel values and their local spatial sensitivities into a single
    geometric object.

\par\endgroup

\subsection{Hamiltonian Dynamics}
\label{sec:fiber-affine-hamiltonian}
We write $J_t$ for an image trajectory, $F_{J_t}=\mathcal F_N(J_t)$
for its continuous representation, and $J^{(m)}$ for numerical iterates.

\begingroup
\setlength{\parskip}{2pt plus 1pt}
\setlength{\abovedisplayskip}{3pt plus 1pt minus 1pt}
\setlength{\belowdisplayskip}{3pt plus 1pt minus 1pt}
\setlength{\abovedisplayshortskip}{0pt plus 1pt}
\setlength{\belowdisplayshortskip}{2pt plus 1pt minus 1pt}
\setlength{\jot}{1pt}
\setlength{\topsep}{3pt plus 1pt minus 1pt}
\setlength{\partopsep}{0pt}
\setlength{\itemsep}{1pt plus 0.5pt}
\setlength{\parsep}{1pt}
\paragraph{Hamiltonian parameterization.}
Let $z=(x,a;\xi,\eta)\in T^*Q$ denote the canonical 
coordinates. For a Hamiltonian, 
let $X_{H_t}$ be the Hamiltonian vector field defined by
$\iota_{X_{H_t}}\omega=d_zH_t$.
A Hamiltonian trajectory
$z(t)=(x(t),a(t);\xi(t),\eta(t))$ with initial point $z_0\in T^*Q$
satisfies
\begin{equation}
\label{equ:hamiltonian_flow} 
\frac{d}{d t}z(t)=X_{H_t}\bigl(z(t)\bigr),
\qquad z(0)=z_0.
\end{equation}

A Hamiltonian flow preserves the exactness of a Lagrangian submanifold
\citep[Proposition~4.13]{viterbo2023generating},
but it need not preserve its representation as a graph over the base
\citep[Lemma~2.4 and Corollary~2.5]{kragh2026generating}.
This raises a natural question: which Hamiltonians transport a prescribed
family of exact Lagrangian graphs? The following lemma provides a precise
characterization: such transport occurs if and only if the generating
functions satisfy the associated Hamilton--Jacobi equation up to an
additive function of time. This characterization is central to our
construction.

\begin{lemma}[Hamiltonian transport of image graphs]
    \label{lem:image-graph-hamiltonians}   
    Let $\mathcal I\subset\mathbb R$ be an open time interval and let
    $F\in C^\infty(\mathcal I\times\mathbb R^2;\mathbb R^n)$.
    Write $F_t(x)=F(t,x)$ and define
    \[
    S_t(x,a)=a^\top F_t(x),
    \qquad
    L_t=\operatorname{graph}(dS_t)
    =
    \left\{
    (x,a;DF_t(x)^\top a,F_t(x)):(x,a)\in Q
    \right\}.
    \]
    Let $H_t$ be a smooth time-dependent Hamiltonian on $T^*Q$.
    Assume that its Hamiltonian flow $\Phi^H_{s\rightarrow t}$
    is defined on all of $T^*Q$ for every $s,t\in\mathcal I$.
    
    Then the following statements are equivalent:
    \begin{enumerate}
        \item
        The Hamiltonian flow transports the prescribed image graphs: $    \Phi^H_{s\rightarrow t}(L_s)=L_t$ for 
        $\qquad s,t\in\mathcal I.$

        \item
        The generating function $S_{t}$ and Hamiltonian $H_{t}$ satisfy the Hamilton--Jacobi equation 
        \begin{equation}
        \label{eq:lemma-hamilton-jacobi}
            \partial_tS_t(q)+H_t(q,d_qS_t(q))=c(t),
            \qquad q\in Q,
        \end{equation}
        where $c\in C^\infty(\mathcal I;\mathbb R)$ only depends on the time. 
        \item
        There exist smooth maps $ U:\mathcal I\times T^*Q\to\mathbb R^2$ and $B:\mathcal I\times T^*Q\to\mathbb R^n$
        such that
        \begin{equation}
        \label{eq:lemma-general-hamiltonian}
        \begin{aligned}
            H_t(x,a;\xi,\eta)
            ={}&c(t)-a^\top\partial_tF_t(x) +\bigl(\xi-DF_t(x)^\top a\bigr)^\top
            U_t(x,a;\xi,\eta)\\
 &+\bigl(\eta-F_t(x)\bigr)^\top
            B_t(x,a;\xi,\eta).
        \end{aligned}
        \end{equation}
    \end{enumerate}
    \end{lemma}

The proof is given in Appendix~\ref{app:image-graph-hamiltonians-proof}.

The restricted coefficients $U_t^L,B_t^L$ and the corresponding
Hamiltonian characteristics are given in
Proposition~\ref{prop:image-graph-characteristics} in
Appendix~\ref{app:image-graph-hamiltonians-proof}.
For a finite-image trajectory, we apply these results with
$F_t=F_{J_t}$ to derive the induced image velocity from
the equation for $\eta$.
\par\endgroup

\paragraph{Induced image dynamics.}
Let $J_t\in\mathcal{X}_N$ be a smooth image path and set
$F_{J_t}=\mathcal{F}_N(J_t)$. Suppose that the Hamiltonian flow
transports the graphs $L_{J_t}$, as characterized by
Lemma~\ref{lem:image-graph-hamiltonians}.
By Hamilton's equation $\dot\eta=-\partial_a H_t$ (\eqref{eq:app-hamilton-equations}), we define 
\[
R_t^L(x,a)
:=-\partial_a H_t
\bigl(x,a;DF_{J_t}(x)^\top a,F_{J_t}(x)\bigr).
\]
Along a Hamiltonian trajectory,
$\eta_t=F_{J_t}(x_t)$ and
$\dot x_t=U_t^L(x_t,a_t)$.
Proposition~\ref{prop:image-graph-characteristics} therefore gives
\begin{equation}
    R_t^L(x_t,a_t)
=\dot\eta_t
=\frac{d}{dt}F_{J_t}(x_t)
=\partial_tF_{J_t}(x_t)
 +DF_{J_t}(x_t)U_t^L(x_t,a_t).
 \label{equ:formula_for_R}
\end{equation}
Equivalently, for every $(x,a)\in Q$,
\[
\partial_tF_{J_t}(x)
=R_t^L(x,a)-DF_{J_t}(x)U_t^L(x,a).
\]
Although the two terms may depend on $a$, their difference is independent
of $a$ by \eqref{equ:formula_for_R}.
For the parameterization used below, the probe remains constant and
spatial transport is independent of the probe:
\[
U_t^L(x,a)=U_t(x),\qquad B_t^L(x,a)=0.
\]
It follows that $R_t^L$ is also independent of $a$; write it as $R_t(x)$.
The induced continuous image velocity is
\begin{equation}
 V_t:=\partial_tF_{J_t}
 =R_t-dF_{J_t}(U_t)=R_t-DF_{J_t}U_t.
 \label{eq:continuous-transport-source}
\end{equation}

\Eqref{eq:continuous-transport-source} is a vector-valued
transport equation with a source term, governing the continuous image
representation. Accordingly, we refer to
$U_t\in\Gamma(TM)$ as the \emph{transport vector field} and
$R_t\in C^\infty(M;V)$ as the \emph{source field}, while
$DF_{J_t}$ denotes the spatial Jacobian. As expressed by \eqref{equ:formula_for_R}, the channel values
change at rate $R_t$ along trajectories satisfying
$\dot x_t=U_t(x_t)$.

Restricting to the pixel grid yields
\begin{equation}
\begin{aligned}
 u_t :=U_t|_{\mathcal G_N},\qquad r_t:=R_t|_{\mathcal G_N}, \qquad
 v_t:=\dot J_t=r_t-\bigl(DF_{J_t}|_{\mathcal G_N}\bigr)u_t.
\end{aligned}
\label{eq:discrete-trajectory-fields}
\end{equation}
All Jacobian--vector products on the grid are pointwise.
Lower-case $u_t,r_t,v_t$ denote grid arrays, whereas upper-case
$U_t,R_t,V_t$ denote continuous fields.

\paragraph{Scope.}
Every $C^1$ image path admits a Hamiltonian realization for any transport
field (Proposition~\ref{prop:hamiltonian-lift}); standard flow matching is
the case $U_t\equiv 0$, and the split is not unique
(Remark~\ref{app:lift-nonuniqueness}). The geometry thus does not restrict
image velocities; it derives the transport--source parameterization, with
the probe $a$ packaging all channels into one generating function
$S_I=a^\top F_I$ governed by one Hamiltonian.

\section{Models}
From \eqref{eq:continuous-transport-source}, Hamiltonian dynamics
naturally induce a decomposition into spatial transport and changes
in channel values. The transport vector field moves image structures,
while the source field changes their appearance along the motion.

Figure~\ref{fig:source-transport-roles} illustrates these complementary
roles: the chosen source field changes the digit's color, while
transport moves its spatial structure.

\begin{figure}[!htbp]
    \centering
    \includegraphics[width=\linewidth]{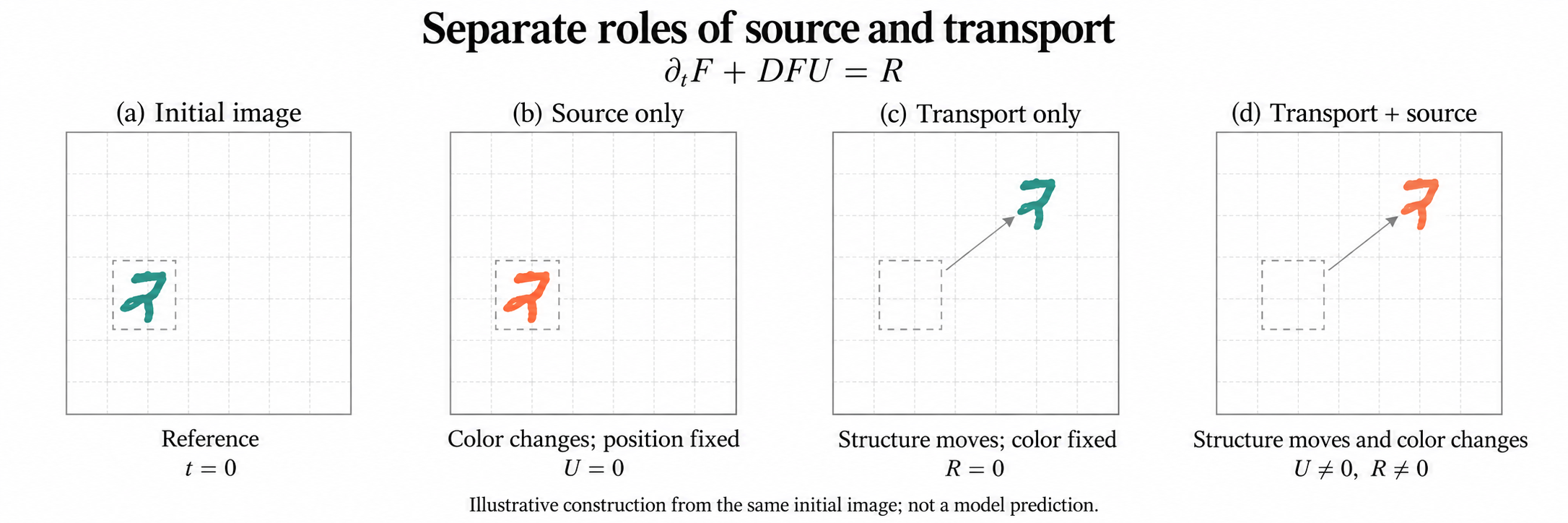}
    \caption{Separate roles of source and transport.}
    \label{fig:source-transport-roles}
\end{figure}

Before introducing the models, we illustrate the benefits of such  decomposition with a simple example in Figure~\ref{fig:transport-source-illustration}.
Pixelwise linear interpolation blends features at fixed spatial
locations, producing overlapping facial features when the endpoints
are not aligned. In contrast, the  transport--source
construction aligns corresponding features while changing their
appearance. This comparison illustrates how the Hamiltonian dynamics can capture structural changes through explicit spatial transport,
motivating its use in image generation and video prediction.

\begin{figure}[!htbp]
    \centering
    \includegraphics[width=\linewidth]
        {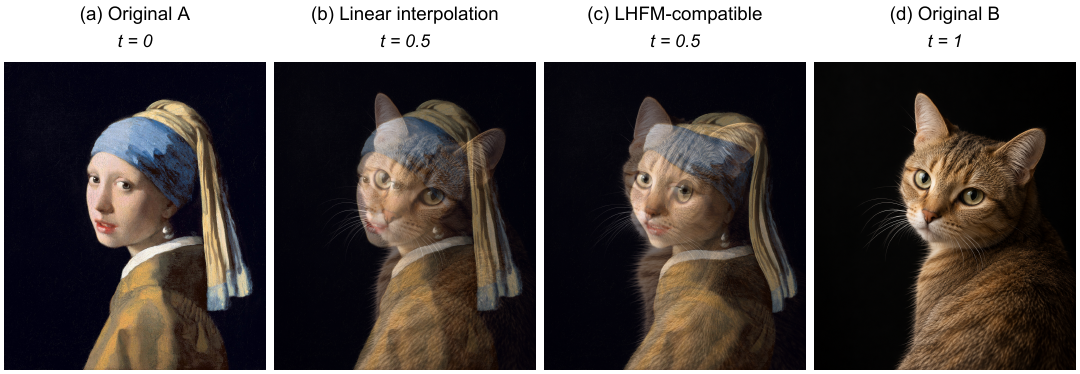}
    \caption{Spatial transport versus pixelwise blending.
    (a,d) Endpoint images: \emph{Girl with a Pearl Earring} and an
    AI-generated cat portrait.
    (b) Pixelwise linear interpolation at $t=0.5$.
    (c) Interpolation compatible with the LHFM
    transport--source formulation at the same time.}
    \label{fig:transport-source-illustration}
\end{figure}

\subsection{LHFM-I: Flow Matching for Image Generation}
\label{sec:image-generation}

Let $\nu_{\mathrm{data}}$ be the image data distribution and
$\nu_0=\mathcal N(0,\mathrm{Id}_d)$ the prior.
For independent $\epsilon\sim\nu_0$, $I\sim\nu_{\mathrm{data}}$,
and $t\sim\mathcal U(0,1)$, define the conditional interpolation
\begin{equation}
 J_t^{\mathrm{cond}}=(1-t)\epsilon+tI,\qquad y=I-\epsilon.
 \label{eq:training-bridge}
\end{equation}

\par\begingroup
\setlength{\parskip}{3pt plus 1pt}
\setlength{\abovedisplayskip}{4pt plus 1pt minus 1pt}
\setlength{\belowdisplayskip}{4pt plus 1pt minus 1pt}
\noindent
\begin{minipage}[t]{0.53\linewidth}
\vspace{0pt}
A neural network with parameters $\theta$ parameterizes a source array
$r_\theta(t,J)$ and a
continuous transport vector field $U_\theta(t,J,\cdot)$. Its grid
restriction is $
 u_\theta(t,J):=\left.U_\theta(t,J,\cdot)\right|_{\mathcal G_N}.
$
The induced image velocity is
\[
 v_\theta(t,J)=r_\theta(t,J)
 -\bigl(DF_J|_{\mathcal G_N}\bigr)u_\theta(t,J).
\]
We minimize the conditional flow-matching objective
\begin{equation}
 \mathcal L_{\mathrm{CFM}}(\theta)
 =\frac1d\mathbb E\left\|
 v_\theta(t,J_t^{\mathrm{cond}})-y\right\|_2^2.
 \label{eq:lhfm-cfm-loss}
\end{equation}
This is the standard independent-endpoint regression objective
\citep{lipman2023flow,tong2024improving}, with a transport--source
parameterization of the predicted image velocity.

For sampling, draw $J_0\sim\nu_0$ and solve
$\dot J_t=v_{\bar\theta}(t,J_t)$, $0\leq t\leq1$,
using exponential moving average (EMA) parameters $\bar\theta$.
The implementation evaluates the full-band DCT spatial Jacobian and
uses a fixed-step Heun solver in image coordinates.
\end{minipage}\hfill
\begin{minipage}[t]{0.44\linewidth}
\vspace{0pt}
\setlength{\intextsep}{0pt}
\algrenewcommand\algorithmicindent{0.8em}
\begin{algorithm}[H]
    \small
    \raggedright
    \caption{LHFM-I for Image Generation}
    \label{alg:lhfm-training}
    \begin{algorithmic}
    \Statex \textbf{Input:} Prior $\nu_0$, data distribution
    $\nu_{\mathrm{data}}$, and parameterized transport and source
    arrays $u_\theta,r_\theta$.
    \Statex \textbf{Velocity:}
    \Statex $\begin{aligned}
    v_\theta(t,J)&:=r_\theta(t,J)\\
    &\quad-\bigl(DF_J|_{\mathcal G_N}\bigr)u_\theta(t,J),
    \end{aligned}$
    \Statex where $F_J=\mathcal F_N(J)$.
    \While{training}
        \State Sample independently
        $\epsilon\sim\nu_0$, $I\sim\nu_{\mathrm{data}}$,
        $t\sim\mathcal U(0,1)$
        \State $J\gets(1-t)\epsilon+tI$
        \State $\mathcal L(\theta)\gets
        d^{-1}\|v_\theta(t,J)-(I-\epsilon)\|_2^2$
        \State $\theta\gets
        \operatorname{Update}(\theta,\nabla_\theta\mathcal L(\theta))$
    \EndWhile
    \State \Return $v_\theta$
    \end{algorithmic}
\end{algorithm}
\end{minipage}
\par\endgroup

\subsection{LHFM-V: Deterministic Video Prediction}
\label{sec:video-prediction}

Let $\mathbf I_{\mathrm{obs}}=(I_1,\ldots,I_{T_{\mathrm{obs}}})$ be
the observed frames. We predict the next $K$ frames by evolving an image
state from $J_{T_{\mathrm{obs}}}=I_{T_{\mathrm{obs}}}$.
Here $s$ denotes video time, measured in frame intervals, rather than
the noise-to-data time $t$ in conditional flow matching.
The continuous formulation is
\begin{equation}
 \partial_sF_{J_s}=R_s-DF_{J_s}U_s,
 \qquad J_{T_{\mathrm{obs}}}=I_{T_{\mathrm{obs}}},
 \label{eq:video-continuous-dynamics}
\end{equation}
with fields conditioned on $\mathbf I_{\mathrm{obs}}$ and the evolving
image and recurrent states.

\paragraph{Conditional prediction and numerical integration.}
LHFM-V encodes the observed frames and their temporal
differences with a convolutional history encoder. It jointly predicts
features for the future frames, then recurrently updates an image state
using transport and source outputs. A causal memory stores the observed
history and recent predicted states; feedback from the current image
updates the recurrent state at each half-frame interval. The field
prediction modules share parameters across the rollout. Unlike the
bounded transport used in LHFM-I, the LHFM-V transport
output has no explicit amplitude bound or temporal gate.
The Moving MNIST implementation uses $T_{\mathrm{obs}}=K=10$ and two
steps per future frame. Set $h=1/2$, $s_m=T_{\mathrm{obs}}+mh$, and
$J^{(0)}=I_{T_{\mathrm{obs}}}$.

For video prediction we use a second-order upwind spatial discretization,
not the DCT derivative used for image generation. The network outputs
$w_m$, in pixels per frame, and $r_m^{\mathrm{net}}$, a source array.
The corresponding normalized-coordinate transport is $u_m=w_m/N$.
Let $A(w_m)$ denote the upwind approximation to $-U_{s_m}\cdot\nabla$,
with zero exterior values. After spatial discretization of
\eqref{eq:video-continuous-dynamics}, we hold the predicted transport
and source fields fixed over each interval of length $h$.
Solving the resulting linear ODE by the variation-of-constants
formula gives
\begin{equation}
 J^{(m+1)}\approx e^{hA(w_m)}J^{(m)}
 +\int_0^h e^{(h-\sigma)A(w_m)}r_m^{\mathrm{net}}\,d\sigma,
 \qquad \widehat I_{T_{\mathrm{obs}}+k}=J^{(2k)}.
 \label{eq:video-exponential-step}
\end{equation}
The fields are recomputed from the updated recurrent state at every half-step;
the source is integrated along with transport, not added only at the
end of the interval. Here $r^{\mathrm{net}}$ distinguishes the numerical
source output from $r=R|_{\mathcal G_N}$ in the exact DCT-based
\eqref{eq:discrete-trajectory-fields}.
Their relation is given in Appendix~\ref{app:fiber-affine-flow}.

\section{Experiments}
\label{sec:experiments}

We evaluate LHFM-I on unconditional CIFAR-10 generation and LHFM-V on
deterministic Moving MNIST prediction. Architectures, training settings, and
evaluation protocols are provided in
Appendix~\ref{app:supplementary-experiments}.

\subsection{Image Generation}
\label{sec:image-generation-experiments}

Table~\ref{tab:cifar10-published-comparison} compares  LHFM-I
with standard independent conditional flow matching (I-CFM)
\citep{tong2024improving} and published baselines.
At 150k updates,  LHFM-I achieves an FID of 3.5809, compared with
3.8241 for the matched I-CFM baseline, a 6.36\% reduction in this comparison. The gap persists across sampling budgets:
at 64, 128, and 256 NFE, LHFM-I lowers FID by 0.28, 0.28, and 0.24,
respectively (Figure~\ref{fig:nfe-fid}). Published results are included for context
rather than as matched-protocol comparisons.

\begin{table}[!htbp]
    \makeatletter
    \let\refstepcounter\H@refstepcounter
    \makeatother
    \centering
    \small
    \begin{minipage}[t]{0.40\linewidth}
    \vspace{0pt}
    \setlength{\abovecaptionskip}{0pt}
    \setlength{\belowcaptionskip}{4pt}
    \raggedright
    \caption[Unconditional CIFAR-10 generation.]{\raggedright
    Unconditional CIFAR-10 generation.
    Published results are taken from Table~5 of
    \citet{tong2024improving}.
    The bottom two rows report our matched 150k-update experiments;
    boldface highlights LHFM-I.}
    \label{tab:cifar10-published-comparison}
    {\footnotesize\raggedright
    \textit{Note.} Published results use different training and evaluation
    protocols; details are given in Appendix~\ref{app:evaluation-protocol}.
    A dash denotes an unreported value.\par}
    \end{minipage}\hfill
    \begin{minipage}[t]{0.58\linewidth}
    \vspace{0pt}
    \centering
    \setlength{\tabcolsep}{3.5pt}
    \renewcommand{\arraystretch}{1.02}
    \begin{tabular*}{\linewidth}{@{}l@{\extracolsep{\fill}}rr@{}}
        \toprule
        \textbf{Method} & \textbf{FID} $\downarrow$ & \textbf{NFE} \\
        \midrule
        DDPM (reported) & 7.48 & 274 \\
        OT-FM (reported) & 6.35 & 142 \\
        VP-FM (reported) & 8.06 & 183 \\
        S.I. (reported) & 10.27 & --- \\
        OT-FM (reproduced) & 11.527 & 139.83 \\
        VP-FM (Tong et al.) & 4.335 & 525.92 \\
        OT-FM (Tong et al.) & 3.655 & 143.00 \\
        S.I. (Tong et al.) & 4.009 & 146.12 \\
        I-CFM (Tong et al.) & 3.659 & 146.42 \\
        OT-CFM (Tong et al.) & 3.577 & 133.94 \\
        \midrule[0.8pt]
        I-CFM (matched baseline) & 3.8241 & 256 \\
        \textbf{LHFM-I (ours)} & \textbf{3.5809} & \textbf{256} \\
        \bottomrule
    \end{tabular*}
    \end{minipage}
\end{table}

\begin{figure}[!htbp]
    \centering
    \begin{minipage}[c]{0.50\linewidth}
        \centering
        \includegraphics[width=\linewidth]{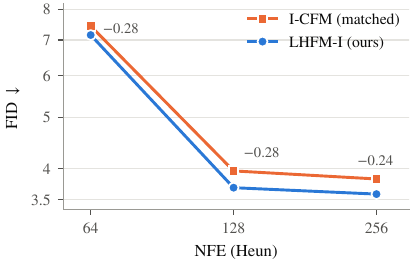}
    \end{minipage}\hfill
    \begin{minipage}[c]{0.46\linewidth}
        \caption{FID versus sampling budget for the matched 150k-update
        models on unconditional CIFAR-10 (fixed-step Heun solver). Labels give the FID reduction of LHFM-I
        relative to I-CFM at each budget.}
        \label{fig:nfe-fid}
    \end{minipage}
\end{figure}

\subsection{Video Prediction}
\label{sec:video-experiments}

We consider Moving MNIST \citep{srivastava2015unsupervised} prediction
from ten observed to ten future frames.
Table~\ref{tab:video-prediction-results} compares LHFM-V with published
recurrent video predictors in terms of model size, computational cost,
and prediction accuracy.
These results provide context rather than protocol-matched comparisons.

\begin{table}[H]
    \makeatletter
    \let\refstepcounter\H@refstepcounter
    \makeatother
    \centering
    \footnotesize
    \setlength{\tabcolsep}{2.5pt}
    \renewcommand{\arraystretch}{1.06}
    \caption{Moving MNIST results for recurrent video predictors, which capture temporal dependencies through recurrent hidden-state updates. Baseline numbers are as reported in the original papers.}
    \label{tab:video-prediction-results}
    \begin{tabular*}{\linewidth}{@{}l@{\extracolsep{\fill}}rrrrr@{}}
        \toprule
        \textbf{Method} & \shortstack{\textbf{Params}\\(M)} &
        \shortstack{\textbf{FLOPs}\\(G)} &
        \textbf{MSE} $\downarrow$ & \textbf{MAE} $\downarrow$ &
        \textbf{SSIM} $\uparrow$ \\
        \midrule
        ConvLSTM~\citep{shi2015convolutional} & 15.0 & 56.8 & 103.3 & 182.9 & 0.707 \\
        PredRNN~\citep{wang2017predrnn} & 23.8 & 116.0 & 56.8 & 126.1 & 0.867 \\
        PredRNN++~\citep{wang2018predrnnpp} & 38.6 & 171.7 & 46.5 & 106.8 & 0.898 \\
        MIM~\citep{wang2019memory} & 38.0 & 179.2 & 44.2 & 101.1 & 0.910 \\
        E3D-LSTM~\citep{wang2019eidetic} & 51.0 & 298.9 & 41.3 & 86.4 & 0.910 \\
        PhyDNet~\citep{leguen2020disentangling} & 3.1 & 15.3 & 24.4 & 70.3 & 0.947 \\
        MAU~\citep{chang2021mau} & 4.5 & 17.8 & 27.6 & 86.5 & 0.937 \\
        PredRNNv2~\citep{wang2022predrnnv2} & 24.6 & 708.0 & 48.4 & 129.8 & 0.891 \\
        SwinLSTM~\citep{tang2023swinlstm} & 20.2 & 69.9 & 17.7 & --- & 0.962 \\
        \midrule[0.8pt]
        \textbf{LHFM-V (ours)} & 18.6 & \textbf{13.1} & 18.6 & 62.5 & 0.958 \\
        \bottomrule
    \end{tabular*}

\end{table}

\paragraph{Why recurrent predictors.}
LHFM-V shares the defining structure of recurrent predictors: a shared
cell advances an explicit state, here the image itself, frame by frame,
so each future frame depends only on the observed frames and previously
predicted states. We therefore compare within this category.

\paragraph{A discussion of computational cost.}
Among the recurrent models in Table~\ref{tab:video-prediction-results}, LHFM-V has the lowest reported FLOP count (13.1G FLOPs). This corresponds to its image-update mechanism: instead of repeatedly reconstructing frames from latent features, LHFM-V predicts transport and source fields that describe motion and color changes. A numerical solver then evolves the current image under these fields, replacing high-resolution neural image decoding with lightweight field prediction and structured image updates. 

\section{Conclusion}
We introduce a geometric framework to image generation and video prediction that
represents images as exact Lagrangian graphs and uses their
Hamiltonian evolution to formulate transport--source dynamics.
This construction connects a geometric representation of images
to a recurrent prediction model: learned transport and source
fields jointly advance the image through numerical integration.
Among the compared recurrent predictors, LHFM-V attains the lowest
FLOP count at similar prediction accuracy. The same geometric
formulation also supports flow matching for image generation,
where our matched experiment shows improved performance over
the baseline.

The current model has not yet undergone systematic
optimization of its architecture or training strategy.
We expect further refinement of these components to improve
the accuracy--efficiency trade-off, a direction that remains
to be validated in future work.

\subsection*{AI use statement}
Generative AI tools were used to assist with language editing,
mathematical checks, software development, and
feedback on experimental design. An AI-generated portrait was used in an illustrative figure, not as training or evaluation data.
The author reviewed the AI-assisted material and takes
responsibility for the final content and claims.

\subsection*{Ethics statement}
This work studies image generation and video prediction using public image and video benchmarks and does not involve human participants or private data. As with other image generators, downstream deployment may reproduce dataset biases or enable misleading synthetic content; such uses require application-specific evaluation and safeguards.

\subsection*{Reproducibility statement}
Section~\ref{sec:image-lagrangian-fm} specifies the image representation,
Hamiltonian dynamics, and separate objectives for image generation and
video prediction. Appendix~\ref{app:lhfm-proofs} contains the proofs and
the relation between continuous geometry and numerical discretization.
Appendix~\ref{app:supplementary-experiments} records model architectures,
training settings, data splits, checkpoint selection, and metric definitions.

The supplementary materials include the implementation,
run-specific configurations, checkpoint identifiers, and
evaluation reports for the reported experiments.
They also provide the numerical integration and metric
computation routines needed to reproduce the evaluations.

\bibliography{iclr2027_conference}
\bibliographystyle{iclr2027_conference}

\appendix
\clearpage
\section*{Appendix Contents}

\begingroup
\small
\setlength{\parskip}{3pt}
\newcommand{\lhfmappendixentry}[2]{%
  \noindent\hspace*{#1}%
  \hyperref[#2]{\ref*{#2}\quad\nameref*{#2}}%
  \nobreak\dotfill\hyperref[#2]{\pageref*{#2}}\par}
{\bfseries\lhfmappendixentry{0pt}{app:mathematical-background}}
\lhfmappendixentry{1em}{app:dct-interpolation}
\lhfmappendixentry{1em}{app:exactness-proof}
\lhfmappendixentry{1em}{app:image-graph-hamiltonians-proof}
\lhfmappendixentry{1em}{app:fiber-affine-flow}
\medskip
{\bfseries\lhfmappendixentry{0pt}{app:supplementary-experiments}}
\lhfmappendixentry{1em}{app:model-hyperparameters}
\lhfmappendixentry{1em}{app:video-hyperparameters}
\lhfmappendixentry{1em}{app:implementation-details}
\lhfmappendixentry{1em}{app:evaluation-protocol}
\medskip
{\bfseries\lhfmappendixentry{0pt}{app:examples}}
\lhfmappendixentry{1em}{app:examples-image}
\lhfmappendixentry{1em}{app:examples-video}
\endgroup
\clearpage

\section{Mathematical Background and Proofs}
\label{app:mathematical-background}
\label{app:lhfm-proofs}

This appendix contains the results needed for the image representation
and Hamiltonian dynamics in Section~\ref{sec:image-lagrangian-fm}.
Interpolation and exact recovery are established first, followed by
the graph-transport proofs and the relation between continuous and
discrete image dynamics.

\subsection{Full-band interpolation and exact recovery}
\label{app:dct-interpolation}

Let $M=\mathbb{R}^2$ be the spatial coordinate space and $V=\mathbb{R}^n$ the
space of $n$-channel pixel values.  For images of resolution $N\times N$, the
pixel-centered grid and image state space are
\begin{equation}
 \mathcal{G}_N =\left\{\left(\frac{j+1/2}{N},\frac{\ell+1/2}{N}\right):
                  j,\ell=0,\ldots,N-1\right\}, \quad  \mathcal{X}_N =V^{\mathcal{G}_N}\cong\mathbb{R}^{nN^2}.
 \label{eq:pixel-state-space}
\end{equation}
Thus an image $I\in\mathcal{X}_N$ assigns a channel vector to each
pixel center.  We write $d=nN^2$; our CIFAR-10 implementation uses
$n=3$ and $N=32$, while Moving MNIST uses $n=1$ and $N=64$.

The following theorem records the standard full-band interpolation
construction based on the orthonormal type-II discrete cosine transform
(DCT-II) \citep{ahmed1974discrete,strang1999discrete}. We record it here to fix the conventions, the proof is standard and thus omitted. 

\begin{theorem}[Exact DCT interpolation]
\label{thm:dct-interpolation}
Let $n,N\geq 1$ be integers and let
$I\in\mathbb{R}^{n\times N\times N}$.  On the pixel-centered grid
$x_j=(j+\tfrac{1}{2})/N$, $j=0,\ldots,N-1$, define
\begin{equation}
 b_0(x)=N^{-1/2},\qquad
 b_k(x)=\sqrt{\frac{2}{N}}\cos(\pi kx),\quad k=1,\ldots,N-1.
 \label{eq:cosine-basis}
\end{equation}
Set $B=(b_k(x_j))_{j,k=0}^{N-1}\in\mathbb{R}^{N\times N}$.
For each channel $c$, let $I_c$ be its pixel array, set
$C_c=B^\top I_c B$, and define, for $x=(x_1,x_2)\in\mathbb{R}^2$,
\begin{equation}
 F_I^c(x)=\sum_{k,\ell=0}^{N-1}(C_c)_{k\ell}\,b_k(x_1)b_\ell(x_2),
 \qquad c=1,\ldots,n.
 \label{eq:cosine-extension}
\end{equation}
Then $B$ is orthogonal, and the map
$F_I=(F_I^1,\ldots,F_I^n)^\top:\mathbb{R}^2\to\mathbb{R}^n$
is smooth, even, and $2$-periodic in each spatial coordinate.  It
interpolates the image exactly:
\begin{equation}
 F_I^c(x_j,x_\ell)=(I_c)_{j\ell},\qquad
 c=1,\ldots,n,\quad j,\ell=0,\ldots,N-1.
 \label{eq:dct-exact-recovery}
\end{equation}
In particular, $I\mapsto F_I$ is linear and injective, with all DCT
coefficients retained, including the constant component.
\end{theorem}

The even, $2$-periodic boundary convention is part of the chosen
representation, rather than an assumption about the unknown underlying
continuum data distribution.

\label{app:functional-background}
\begin{definition}[Image function space]
\label{def:image-function-space-topology}
For $m\geq1$, let
\[
 \mathcal H_N^{(m)}
 =\operatorname{span}\{b_k(x_1)b_\ell(x_2)e_c:
    0\leq k,\ell<N,\ 1\leq c\leq m\}
 \subset C^\infty(M;\mathbb R^m),
\]
with the subspace topology inherited from the usual Fr\'echet topology.
For $m=n$, write $\mathcal H_N=\mathcal H_N^{(n)}$.
The corresponding grid space is
$\mathcal X_N^{(m)}=(\mathbb R^m)^{\mathcal G_N}$.
\end{definition}

\begin{proposition}[Extension, recovery, and differentiation]
\label{prop:state-space-isomorphism}
The componentwise extension
$\mathcal F_N^{(m)}:\mathcal X_N^{(m)}\to\mathcal H_N^{(m)}$
is a topological linear isomorphism.  In particular, for $m=n$,
\begin{equation}
 \mathcal{F}_N:\mathcal{X}_N\xrightarrow{\;\cong\;}\mathcal{H}_N,
 \qquad \mathcal{F}_N(I)=F_I,
 \label{eq:image-state-chart}
\end{equation}
with inverse grid sampling $\mathcal{S}_N(F)=F|_{\mathcal{G}_N}$.
For a $C^1$ image path $J_t$,
\begin{equation}
 \partial_t F_{J_t}
 =D\mathcal F_N(J_t)[\dot J_t]
 =\mathcal F_N(\dot J_t).
 \label{eq:linear-extension-tangent-map}
\end{equation}
\end{proposition}
\begin{proof}
Theorem~\ref{thm:dct-interpolation} proves injectivity and exact recovery.
The tensor-product basis spans $\mathcal H_N^{(m)}$, so grid sampling
also gives a right inverse. The extension is continuous because it is a
finite expansion in fixed smooth functions, and point evaluation is
continuous in the Fr\'echet topology. Linearity implies
$D\mathcal F_N(J)[h]=\mathcal F_N(h)$, and the chain rule gives the
path identity. All these statements concern a fixed resolution $N$.
\end{proof}

The spatial differential is a different map:
$dF_J|_x:T_xM\to V$ acts on spatial tangent vectors, whereas
$D\mathcal F_N(J)$ acts on image-array perturbations.
For a spatial path $x_t$, the ordinary chain rule gives
$\frac{d}{dt}F_{J_t}(x_t)=\mathcal F_N(\dot J_t)(x_t)
+DF_{J_t}(x_t)\dot x_t$.

\subsection{Exact Lagrangian Image Graphs}
\label{app:exactness-proof}

Recall that an embedding $\iota:L\hookrightarrow(T^*Q,\omega=-d\lambda)$ is
Lagrangian if $\dim L=\dim Q$ and $\iota^*\omega=0$.
It is exact if $\iota^*\lambda=df$ for a globally defined function
$f:L\to\mathbb R$. In canonical coordinates,
$\lambda=p^\top dq$ and $\omega=\sum_i dq^i\wedge dp_i$.

Our convention $\iota_{X_{H_t}}\omega=d_zH_t$ gives
\begin{equation}
 \dot q=\partial_pH_t,\qquad \dot p=-\partial_qH_t,
 \qquad q=(x,a),\quad p=(\xi,\eta).
 \label{eq:app-hamilton-equations}
\end{equation}
These are ambient derivatives, taken before restricting to a graph.

\begin{proposition}[Exactness and image recovery]
\label{prop:intensity-lift}
For every $I\in\mathcal X_N$, the graph $L_I=\operatorname{graph}(dS_I)$
is an exact Lagrangian submanifold of $T^*Q$. The encoder $E$ is
injective, and $I$ is recovered by reading $\eta$ at $a=0$ and
$x\in\mathcal G_N$.
\end{proposition}
\begin{proof}
For $\iota_I(x,a)=(x,a;DF_I(x)^\top a,F_I(x))$,
\[
 \iota_I^*\lambda
 =a^\top DF_I(x)\,dx+F_I(x)^\top da
 =d\bigl(a^\top F_I(x)\bigr)=dS_I.
\]
The graph is diffeomorphic
to $Q$ and has half the dimension of $T^*Q$.
At $a=0$ and $x\in\mathcal G_N$, one has
$\xi=0$ and $\eta=F_I(x)=I(x)$ by
Theorem~\ref{thm:dct-interpolation}. Thus two coincident graphs
have the same image samples, proving injectivity.
\end{proof}

Constant channel offsets are retained through $\eta$.
Although an individual graph has dimension $n+2$, its image family
has $nN^2$ degrees of freedom; this representation is not a compression
to $n+2$ scalar coordinates. In image generation, the prior is a
distribution over whole image arrays and hence whole encoded graphs,
not a distribution of individual points in $T^*Q$.

\subsection{Hamiltonians transporting image graphs}
\label{app:image-graph-hamiltonians-proof}

    \begin{proof}[Proof of Lemma~\ref{lem:image-graph-hamiltonians}]
    Write $q=(x,a)$ and $p=(\xi,\eta)$.
    The moving graph is specified by
    \[
        p=d_qS_t(q).
    \]
    Using Hamilton's equations
    \[
        \dot q=\partial_pH_t,
        \qquad
        \dot p=-\partial_qH_t,
    \]
    the extended vector field $\partial_t+X_{H_t}$ is tangent
    to this moving graph if and only if
    \[
        -\partial_qH_t(q,d_qS_t)
        =
        \partial_t d_qS_t
        +D_q^2S_t\,\partial_pH_t(q,d_qS_t).
    \]
    By the chain rule, this is equivalent to
    \[
        d_q\!\left[
            \partial_tS_t+H_t(q,d_qS_t)
        \right]=0.
    \]
    Since $Q=\mathbb R^2\times\mathbb R^n$ is connected, the
    expression in brackets is a function of time alone. 
    This proves the equivalence between tangency and
    \eqref{eq:lemma-hamilton-jacobi}.
    
    Under the assumed existence of the Hamiltonian flow,
    tangency and uniqueness of trajectories imply
    \[
        \Phi^H_{s\rightarrow t}(L_s)\subseteq L_t.
    \]
    Applying the inverse flow $\Phi^H_{t\rightarrow s}$ gives
    the reverse inclusion, proving statement~(1) of
    Lemma~\ref{lem:image-graph-hamiltonians}.
    Conversely, differentiating this graph-transport identity gives tangency.
    Thus statements~(1) and~(2) are equivalent.
    
    To prove that statement~(2) implies statement~(3), set
    \[
        p_t(q)=d_qS_t(q),
        \qquad
        K_t(q,p)=H_t(q,p)-c(t)+\partial_tS_t(q).
    \]
    \Eqref{eq:lemma-hamilton-jacobi} gives
    \[
        K_t(q,p_t(q))=0.
    \]
    The fundamental theorem of calculus along the momentum fiber
    therefore yields
    \[
    \begin{aligned}
        K_t(q,p)
        &=
        \int_0^1
        \partial_pK_t\bigl(q,p_t(q)+s[p-p_t(q)]\bigr)^\top
        [p-p_t(q)]\,ds\\
        &=
        [p-p_t(q)]^\top Z_t(q,p),
    \end{aligned}
    \]
    where
    \[
        Z_t(q,p)
        =
        \int_0^1
        \partial_pK_t\bigl(q,p_t(q)+s[p-p_t(q)]\bigr)\,ds
    \]
    is smooth.
    Denote the $\xi$- and $\eta$-components of $Z_t$ by $U_t$ and $B_t$,
    respectively.  Using
    \[
        p_t(q)=\bigl(DF_t(x)^\top a,F_t(x)\bigr)
    \]
    gives \eqref{eq:lemma-general-hamiltonian}.
    Conversely, restricting
    \eqref{eq:lemma-general-hamiltonian} to $L_t$ immediately
    gives \eqref{eq:lemma-hamilton-jacobi}.
    \end{proof}

\begin{proposition}[Hamiltonian characteristics on image graphs]
    \label{prop:image-graph-characteristics}
    Under the assumptions of Lemma~\ref{lem:image-graph-hamiltonians},
    suppose that its equivalent conditions hold, and let $U_t$ and $B_t$
    be the coefficients in \eqref{eq:lemma-general-hamiltonian}.
    Define
    \[
    \begin{aligned}
        U_t^L(x,a)
        &=
        U_t\bigl(x,a;DF_t(x)^\top a,F_t(x)\bigr), \quad 
        B_t^L(x,a)
        =
        B_t\bigl(x,a;DF_t(x)^\top a,F_t(x)\bigr).
    \end{aligned}
    \]
    The Hamiltonian characteristics restricted to $L_t$ are
    \begin{equation}
    \label{eq:lemma-graph-characteristics}
    \begin{aligned}
        \dot x
        &=U_t^L(x,a),\qquad
        \dot a=B_t^L(x,a),\\
        \dot\eta
        &=\partial_tF_t(x)+DF_t(x)U_t^L(x,a),\\
        \dot\xi
        &=\bigl(D\partial_tF_t(x)\bigr)^\top a
          +D_x\bigl(DF_t(x)^\top a\bigr)U_t^L(x,a)
          +DF_t(x)^\top B_t^L(x,a).
    \end{aligned}
    \end{equation}
    Here $D$ denotes differentiation with respect to $x$, and
    $D_x(DF_t(x)^\top a)$ is computed with $a$ held fixed.
    \end{proposition}

    \begin{proof}[Proof of Proposition~\ref{prop:image-graph-characteristics}]
    Differentiate
    \eqref{eq:lemma-general-hamiltonian} in the ambient phase
    space and then restrict to $L_t$.
    All terms multiplied by
    $\xi-DF_t(x)^\top a$ or $\eta-F_t(x)$ vanish, yielding
    \eqref{eq:lemma-graph-characteristics}.
    In particular,
    \[
        \dot\eta
        =
        \frac{d}{dt}F_t(x(t)),
        \qquad
        \dot\xi
        =
        \frac{d}{dt}\bigl[DF_t(x(t))^\top a(t)\bigr],
    \]
    which explicitly verifies preservation of the graph constraints.
    \end{proof}

The function $c(t)$ in the Hamilton--Jacobi condition expresses the
freedom to add a function of time to a generating function:
replacing $S_t$ by $S_t-\int_{t_0}^t c(\sigma)\,d\sigma$ leaves
its graph unchanged and sets the right-hand side to zero.
Graph transport permits reparameterization along the submanifold;
points with fixed base coordinate need not be Hamiltonian trajectories
\citep{carinena2006geometric}.

\subsection{Compatible Continuous and Discrete Image Dynamics}
\label{app:fiber-affine-flow}

\begin{proposition}[Compatible Hamiltonian realization]
\label{prop:hamiltonian-lift}
Let $J_t$ be a $C^1$ image path with $v_t=\dot J_t$.
Let $U_t$ be continuous in time, smooth in space, with jointly
continuous spatial derivatives, and assume its spatial trajectories
exist throughout the interval for every initial time and point.
Define
\begin{equation}
 R_t=\mathcal F_N(v_t)+DF_{J_t}U_t,
 \qquad H_t=\xi^\top U_t-a^\top R_t.
 \label{eq:app-compatible-source}
\end{equation}
Then the Hamiltonian flow transports $L_{J_{t_0}}$ to $L_{J_t}$.
Without interval-wide existence, the statement holds on the
corresponding flow domains.
\end{proposition}
\begin{proof}
Proposition~\ref{prop:state-space-isomorphism} gives
$\partial_tF_{J_t}=\mathcal F_N(v_t)=R_t-DF_{J_t}U_t$.
Hamilton's equations are
\[
 \dot x=U_t(x),\quad \dot a=0,\quad
 \dot\eta=R_t(x),\quad
 \dot\xi=-DU_t(x)^\top\xi+DR_t(x)^\top a.
\]
Along a spatial characteristic, set
$\eta_t=F_{J_t}(x_t)$ and $\xi_t=DF_{J_t}(x_t)^\top a$.
The chain rule verifies the $\eta$ equation.
Differentiating $\partial_tF_{J_t}=R_t-DF_{J_t}U_t$ in $x$
and using the symmetry of second spatial derivatives verifies the
$\xi$ equation. Thus trajectories starting on the graph stay on it.
The spatial flow is a diffeomorphism on the assumed interval;
the remaining equations are linear or affine along it, so both
forward and backward ambient trajectories exist. Applying the inverse
flow proves equality of the transported graphs.
Finally, Cartan's formula gives
$\mathcal L_{X_{H_t}}\lambda=d(\lambda(X_{H_t})-H_t)$.
Integration in time shows that the ambient flow preserves $\lambda$
up to an exact form and therefore preserves $\omega=-d\lambda$.
\end{proof}

\begin{rem}[Spatial discretization and the source field]
\label{rem:discrete-source-compatibility}
For the DCT-based image velocity
$v_t=r_\theta-(DF_{J_t}|_{\mathcal G_N})u_t$, the compatible source
in \eqref{eq:app-compatible-source} satisfies
$R_t|_{\mathcal G_N}=r_\theta$.
However, $R_t$ need not equal $\mathcal F_N(r_\theta)$, since
$DF_{J_t}U_t$ need not belong to $\mathcal H_N$.

For video prediction, substituting
$v_s=r_s^{\mathrm{net}}+A(w_s)J_s$ and $u_s=w_s/N$ into
\eqref{eq:app-compatible-source} gives
\begin{equation}
 r_s:=R_s|_{\mathcal G_N}
 =r_s^{\mathrm{net}}+A(w_s)J_s
  +\bigl(DF_{J_s}|_{\mathcal G_N}\bigr)u_s.
 \label{eq:video-source-compatibility}
\end{equation}
Thus $r_s$ generally differs from $r_s^{\mathrm{net}}$ because the
upwind and DCT transport terms differ. This correction specifies
the compatible continuous Hamiltonian source; it is not added to
the numerical update.
\end{rem}

For the video variant, the predicted fields are frozen
within each half-frame interval and may change
discontinuously at interval boundaries.
Proposition~\ref{prop:hamiltonian-lift} therefore applies
intervalwise to the exact solutions of the frozen-field
grid ODEs, subject to its spatial-flow existence assumptions.
Since the image state is continuous across interval
boundaries, the corresponding Hamiltonian flows can be
composed to transport the image graphs over successive
intervals. The implemented truncated-series updates
approximate these exact grid evolutions.

\begin{rem}[Nonuniqueness and numerical interpretation]
\label{app:lift-nonuniqueness}
\label{app:numerical-closure-proof}
For the DCT velocity, $(u,r)\mapsto(u+\delta u,
r+(DF_J|_{\mathcal G_N})\delta u)$ leaves $v$ unchanged whenever
the modified fields remain admissible. Thus the CFM objective alone
does not uniquely identify the transport and source fields. 
\end{rem}

\clearpage
\section{Supplementary Experiments}
\label{app:supplementary-experiments}

We provide architecture, training, and evaluation details for image
generation and deterministic video prediction. LHFM-I, LHFM-V, and the
matched I-CFM baseline are trained from random initialization, without
pretrained components, distillation, or additional fine-tuning.

\subsection{Image-Generation Architectures and Hyperparameters}
\label{app:model-hyperparameters}

The matched CIFAR-10 models share a time-dependent U-Net backbone and training settings
(Table~\ref{tab:matched150k-hyperparameters}). Standard I-CFM predicts
three image-velocity channels; LHFM-I predicts two transport and
three source channels, adding only 2,306 parameters.
 LHFM-I uses $u_\theta=0.125t\,\tanh(\widetilde u_\theta)$ and
$v_\theta=r_\theta-(DF_J|_{\mathcal G_N})u_\theta$.
Here $\widetilde u_\theta$ is the raw two-component network output.
The elementwise nonlinearity bounds each grid component by $0.125t$
and suppresses transport near the noise endpoint. This is one choice
allowed by Section~\ref{sec:fiber-affine-hamiltonian}.
Both models use independent endpoint sampling with $\sigma=0$ and
the conditional flow-matching objective in Algorithm~\ref{alg:lhfm-training},
with $v_\theta=f_\theta$ for I-CFM. The I-CFM baseline follows
\citet{tong2024improving} and the authors'
\href{https://github.com/atong01/conditional-flow-matching}{TorchCFM}
method, using our local matched backbone rather than the upstream
TorchCFM U-Net or a pretrained checkpoint;
it is not an LHFM-I variant or minibatch OT-CFM.

\begin{table}[!htbp]
    \makeatletter
    \let\refstepcounter\H@refstepcounter
    \makeatother
    \centering
    \small
    \setlength{\tabcolsep}{6pt}
    \renewcommand{\arraystretch}{1.08}
    \begin{tabular}{@{}lcc@{}}
        \toprule
        \textbf{Setting} & \textbf{Standard I-CFM} & \textbf{LHFM-I} \\
        \midrule
        Dataset / resolution & \multicolumn{2}{c}{CIFAR-10 / $32\times32$ RGB} \\
        Conditioning & \multicolumn{2}{c}{Unconditional} \\
        Base channels / multipliers & \multicolumn{2}{c}{128 / $(1,2,2,2)$} \\
        Residual blocks per level (down / up) & \multicolumn{2}{c}{$2/3$} \\
        Attention resolutions & \multicolumn{2}{c}{$16,8$; $4$ (bottleneck)} \\
        Attention heads / channels per head & \multicolumn{2}{c}{$4/64$} \\
        Dropout & \multicolumn{2}{c}{0.1} \\
        Output channels & 3 (velocity) & 5 (transport and source) \\
        Parameters & 39,625,603 & 39,627,909 \\
        \midrule
        Batch size / training seed & \multicolumn{2}{c}{$256/270829$} \\
        Optimizer / coefficients / weight decay & \multicolumn{2}{c}{AdamW / $(0.9,0.999)$ / 0} \\
        Peak / final scheduled learning rate & \multicolumn{2}{c}{$2.5\times10^{-4}/2\times10^{-5}$} \\
        Schedule / warmup & \multicolumn{2}{c}{Cosine decay / 2,000 updates} \\
         evaluated checkpoint & \multicolumn{2}{c}{ 150k updates} \\
        EMA target decay / gradient clipping & \multicolumn{2}{c}{$0.9999/1$} \\
        Network / field precision & \multicolumn{2}{c}{BF16 / FP32} \\
        Data augmentation & \multicolumn{2}{c}{Random horizontal flip} \\
        Training hardware & \multicolumn{2}{c}{One NVIDIA GeForce RTX 5090} \\
        \bottomrule
    \end{tabular}
    \begingroup
    \makeatletter
    \let\refstepcounter\H@refstepcounter
    \makeatother
    \caption{Settings for the matched CIFAR-10 comparison. Shared entries
    apply to both models; reported results use the 150k-update checkpoints.}
    \label{tab:matched150k-hyperparameters}
    \endgroup
\end{table}

Both models are initialized from scratch and sample training images
uniformly with replacement. The EMA decay at update $k$ is
$\min\{0.9999,(1+k)/(10+k)\}$.

\subsection{Video-Prediction Architectures and Hyperparameters}
\label{app:video-hyperparameters}

\paragraph{Architecture.}
In LHFM-V, a convolutional encoder processes the observed frames, consecutive
frame differences, and spatial coordinates. Temporal attention combines
the observed history with six recent predicted states, while a separate
convolutional module predicts features for all ten future frames.
Separate heads with pixel-shuffle readouts produce the source
$r^{\mathrm{net}}$ and transport $w$, without amplitude bounds or temporal
gates. Predicted images are fed back into the recurrent state at each
half-frame step; training uses the full rollout without teacher forcing.
Figure~\ref{fig:lhfm-video-architecture} illustrates the architecture,
and Table~\ref{tab:video-hyperparameters} summarizes its settings.

\begin{figure}[!htbp]
    \centering
    \includegraphics[width=\linewidth,height=0.82\textheight,keepaspectratio]
        {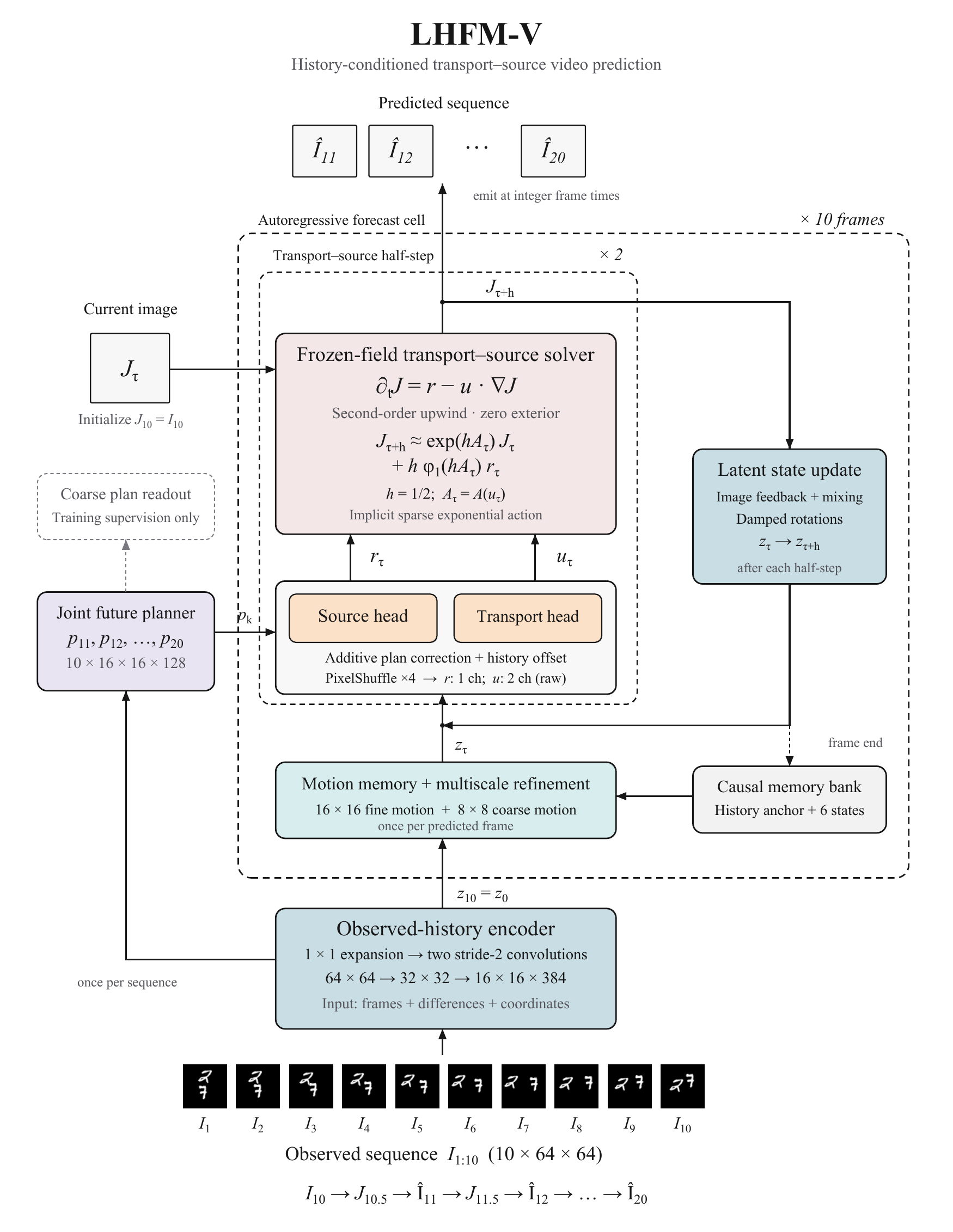}
    \caption{Architecture of LHFM-V for deterministic video prediction.
    Each future frame is produced by two transport--source half-steps,
    with image feedback and recurrent-state updates.
    In the diagram, $\tau$, $u$, and $r$ correspond to $s$, $w$, and
    $r^{\mathrm{net}}$ in the text, respectively; $w$ is measured in pixels
    per frame and satisfies $w=Nu$ in the text's notation.
    The function $\varphi_1(X)=\int_0^1 e^{(1-\sigma)X}\,d\sigma$
    represents the source integral in \eqref{eq:video-exponential-step}.}
    \label{fig:lhfm-video-architecture}
\end{figure}

\paragraph{Training.}
We train from scratch using AdamW on a single NVIDIA
GeForce RTX 5090. For the first 400,000 optimizer updates,
we use a two-phase OneCycle schedule configured for
1,250,000 updates with cosine annealing and
$\mathrm{pct\_start}=0.3$. The learning rate increases from
$4\times10^{-5}$ to $10^{-3}$ at update 375,000.
We then resume from the complete 400,000-update checkpoint,
preserving the model, optimizer, EMA, and random states.
At update 400,001, we halve the learning rate prescribed
by the original schedule to approximately
$4.989935\times10^{-4}$, then apply cosine decay to
$4\times10^{-9}$ at update 1,250,000, without additional
warmup. Adam's $\beta_1$ retains the original OneCycle
trajectory, decreasing from 0.95 to 0.85 during warmup
and subsequently increasing to 0.95; $\beta_2=0.999$.
Neural training uses BF16, while image integration,
recurrent states, and optimization use FP32.
EMA evaluation uses FP32 throughout.

\begin{table}[!htbp]
    \makeatletter
    \let\refstepcounter\H@refstepcounter
    \makeatother
    \centering
    \small
    \setlength{\tabcolsep}{5pt}
    \renewcommand{\arraystretch}{1.04}
    \begin{tabular}{@{}ll@{}}
        \toprule
        \textbf{Setting} & \textbf{LHFM-V} \\
        \midrule
        Resolution / channels & $64\times64$ / 1 \\
        Observed / future frames & $10/10$ \\
        Encoder widths / residual blocks & $(64,128,448)$ / 3 \\
        Recurrent-state channels / resolution & 384 / $16\times16$ \\
        Memory frames / attention heads & 6 plus history / 4 \\
        Memory key / value dimensions & $128/256$ \\
        Fine / coarse spatial channels & $320/192$ \\
        Future-feature channels / blocks & $128/4$ \\
        Source / transport head hidden width & $192/192$ \\
        Field evaluations per future frame & 2 \\
        Parameters (including auxiliary head) & 18,637,493 \\
        \midrule
        Batch size / training seed & $16/270829$ \\
        Optimizer / weight decay & AdamW / $10^{-4}$ \\
        Adam $\beta_1$ range / $\beta_2$ & $[0.85,0.95]$ / 0.999 \\
        Initial / peak learning rate & $4\times10^{-5}/10^{-3}$ \\
        Updates at evaluation / learning rate & 1.25M / $4\times10^{-9}$ \\
        EMA decay / gradient clipping (norm) & $0.999/1$ \\
        \bottomrule
    \end{tabular}
    \caption{Architecture and training settings for LHFM-V on Moving MNIST.}
    \label{tab:video-hyperparameters}
\end{table}

\paragraph{Objective.}
Let $\operatorname{MSE}$ denote the mean squared error over the minibatch,
channels, and spatial positions. With $K=10$ and two half-steps per
future frame, the objective is
\begin{equation}
\begin{aligned}
 \mathcal L_{\mathrm{pred}}
 &=\frac1K\sum_{k=1}^{K}\operatorname{MSE}
     (\widehat I_{T_{\mathrm{obs}}+k},I_{T_{\mathrm{obs}}+k})\\
 &\quad+\frac{\lambda_r}{2K}\sum_{m=0}^{2K-1}
     \operatorname{mean}[(r_m^{\mathrm{net}})^2]
     +\frac{\lambda_u}{2K}\sum_{m=0}^{2K-1}\operatorname{TV}(w_m)\\
 &\quad+\frac{\lambda_c}{K}\sum_{k=1}^{K}\operatorname{MSE}
     (\widehat I_k^{\mathrm{coarse}},\mathcal P_4 I_{T_{\mathrm{obs}}+k}),
 \label{eq:video-prediction-loss}
\end{aligned}
\end{equation}
where $\lambda_r=10^{-3}$, $\lambda_u=10^{-4}$, and $\lambda_c=0.05$.
Here $\mathcal P_4$ is average pooling with kernel and stride 4.
A shared linear head produces the $16\times16$ auxiliary predictions
$\widehat I_k^{\mathrm{coarse}}$ for training only; they do not enter the
image rollout. The transport regularizer is
\[
 \operatorname{TV}(w)=\tfrac12\left(
 \operatorname{mean}|w_{\cdot,j+1,\ell}-w_{\cdot,j,\ell}|
 +\operatorname{mean}|w_{\cdot,j,\ell+1}-w_{\cdot,j,\ell}|
 \right),
\]
with means over the minibatch, both components, and valid adjacent pairs.

\subsection{Numerical Implementation}
\label{app:implementation-details}

\paragraph{Image generation.}
We evaluate the full-band DCT spatial Jacobian on the pixel grid and
integrate $\dot J_t=v_\theta(t,J_t)$ with 128 Heun steps (256 NFE).
The Hamiltonian is not explicitly evaluated during sampling.

\paragraph{Video prediction.}
\label{app:video-implementation}
Each future frame uses two half-frame transport--source updates.
In normalized coordinates, the grid spacing is $\Delta=1/N$;
the implemented transport output $w=Nu$ uses pixels per frame.
Writing $w_i^+=\max(w_i,0)$ and $w_i^-=\max(-w_i,0)$, the stencil is
\begin{equation}
\begin{split}
 [A(w)J](g)=\sum_{i=1}^{2}\bigg\{&w_i^+(g)
 \left[-\tfrac32J(g)+2J(g-\Delta e_i)-\tfrac12J(g-2\Delta e_i)\right]\\
 &+w_i^-(g)
 \left[-\tfrac32J(g)+2J(g+\Delta e_i)-\tfrac12J(g+2\Delta e_i)\right]
 \bigg\}.
\end{split}
\label{eq:video-upwind-stencil}
\end{equation}
We use zero values outside the grid. This second-order upwind stencil
discretizes $-U\cdot\nabla F$, not $-\nabla\cdot(UF)$.

With fields frozen over a half-step of length $h$, we set
\[
q=\frac32\max\left\{1,\max_{b,g}
\bigl(|w_{b,1}(g)|+|w_{b,2}(g)|\bigr)\right\},
\qquad
L=\max\{1,\lceil hq/3\rceil\},
\]
where $b$ indexes the minibatch and $g$ the spatial grid.
Each internal interval has length $\delta=h/L$.
Writing $P=\mathrm{Id}+A(w)/q$ and $\mu=q\delta$,
we compute
\[
z_0=J,\qquad
z_{k+1}=Pz_k+r^{\mathrm{net}}/q,
\qquad k=0,\ldots,23,
\]
and update
\[
J_{\mathrm{next}}
=e^{-\mu}\sum_{k=0}^{24}\frac{\mu^k}{k!}z_k.
\]
We repeat this update over the $L$ internal intervals
using the same predicted fields. The numerical scaling
$q$ and subdivision count $L$ are treated as constants
during differentiation; gradients propagate through
$A(w)$ and $r^{\mathrm{net}}$.

\subsection{Evaluation Protocols}
\label{app:evaluation-protocol}

\paragraph{Image generation.}
We evaluate the 150k-update EMA checkpoints using 50,000 generated
images and all 50,000 CIFAR-10 training images as the reference.
Sampling uses Heun128 (256 NFE) with batch size 64.
FID is computed with TorchMetrics 1.9.0 and
torch-fidelity 0.4.0. Generated pixels are mapped from $[-1,1]$ to $[0,1]$ and
clipped only for evaluation, with no intermediate solver clipping.
The published entries in Table~\ref{tab:cifar10-published-comparison}
use adaptive DOPRI5, unlike our fixed-step Heun evaluations.
Published baselines retain their
original protocols and are not matched-control comparisons.

\paragraph{Video prediction.}
Moving MNIST frames are normalized to $[0,1]$.
We partition the 60,000 MNIST training digit images into
55,000 training digits and 5,000 validation digits using
split seed 271100. Training sequences are generated
on demand from the training digit pool, with each sequence
determined by data seed 270829 and its absolute sequence
index. Each sequence contains two moving digits and
20 frames, split into ten observed and ten future frames.
We construct 1,024 fixed validation sequences from the
held-out digit pool using seed 271109 and select the EMA
checkpoint with the lowest validation MSE. The results in
Table~\ref{tab:video-prediction-results} evaluate the selected 1.25M-step (2000 epochs)
EMA checkpoint on the official 10,000-sequence test set, which was not
used for checkpoint selection.
MSE and MAE are the spatial sums of squared and absolute errors,
respectively, on unclipped predictions, averaged over future frames and
sequences. SSIM and PSNR use predictions
clipped to $[0,1]$. SSIM is computed using the scikit-image 0.19.3
implementation with a $7\times7$ uniform window, sample covariance,
\texttt{data\_range=2}, and $K_1=0.01$, $K_2=0.03$.
The SSIM map is averaged over the interior after excluding a three-pixel
border, then averaged equally over future frames and sequences.
PSNR is $-10\log_{10}(\max\{\mathrm{MSE},10^{-12}\})$ using mean-pixel
MSE per frame, then averaged over frames and sequences.
Each observed sequence produces one deterministic prediction.

\clearpage
\section{Examples}
\label{app:examples}

\subsection{LHFM-I: Image Generation}
\label{app:examples-image}

Figure~\ref{fig:lhfm-i-examples} shows samples generated by LHFM-I on
CIFAR-10.

\begin{figure}[!htbp]
    \makeatletter
    \let\refstepcounter\H@refstepcounter
    \makeatother
    \centering
    \includegraphics[width=0.80\linewidth]
        {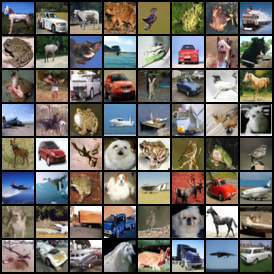}
    \caption{CIFAR-10 samples from LHFM-I, generated using a 150k-update
    EMA checkpoint and 128 Heun steps (256 NFE). The 64-image grid is
    retained in its original order. }
    \label{fig:lhfm-i-examples}
\end{figure}

\clearpage
\subsection{LHFM-V: Video Prediction}
\label{app:examples-video}

Figures~\ref{fig:lhfm-v-examples-first} and~\ref{fig:lhfm-v-examples-second}
show the first six official Moving MNIST test sequences in their original
order (sequence IDs 0--5). Predictions use the same validation-selected
1.25M-update EMA checkpoint as Table~\ref{tab:video-prediction-results}.

\begingroup
\newlength{\lhfmexamplerowheight}
\newcommand{\lhfmexampleclip}[2]{%
    \par\noindent{\small\textbf{Sequence #1}}\par\nobreak\vspace{2pt}%
    \setlength{\lhfmexamplerowheight}{0.084\linewidth}%
    \noindent\begin{minipage}[c]{0.13\linewidth}
        \scriptsize\setlength{\parskip}{0pt}\setlength{\parindent}{0pt}%
        \parbox[c][\lhfmexamplerowheight][c]{\linewidth}{\raggedleft Observed}\par
        \nointerlineskip
        \parbox[c][\lhfmexamplerowheight][c]{\linewidth}{\raggedleft Ground truth}\par
        \nointerlineskip
        \parbox[c][\lhfmexamplerowheight][c]{\linewidth}{\raggedleft LHFM-V}
    \end{minipage}\hfill
    \begin{minipage}[c]{0.84\linewidth}
        \includegraphics[width=\linewidth]{#2}%
    \end{minipage}\par
}

\begin{figure}[!htbp]
    \makeatletter
    \let\refstepcounter\H@refstepcounter
    \makeatother
    \centering
    \lhfmexampleclip{0}{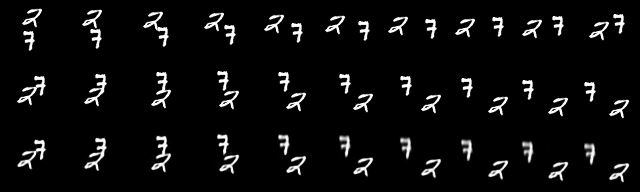}
    \vspace{7pt}
    \lhfmexampleclip{1}{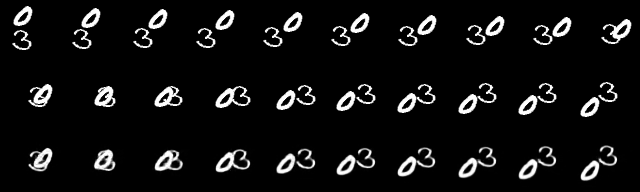}
    \vspace{7pt}
    \lhfmexampleclip{2}{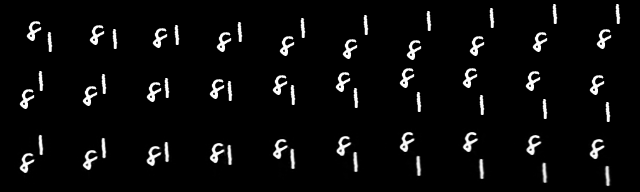}
    \caption{LHFM-V predictions on Moving MNIST (sequence IDs 0--2).
    For each sequence, rows show the ten observed frames (1--10),
    the ten ground-truth future frames (11--20), and the corresponding
    LHFM-V predictions. Time progresses from left to right.}
    \label{fig:lhfm-v-examples-first}
\end{figure}

\clearpage
\begin{figure}[!htbp]
    \makeatletter
    \let\refstepcounter\H@refstepcounter
    \makeatother
    \centering
    \lhfmexampleclip{3}{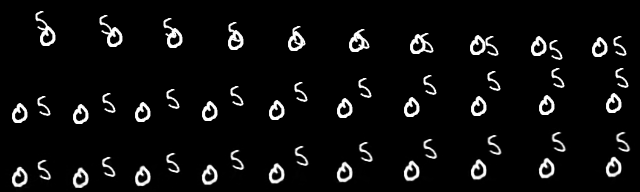}
    \vspace{7pt}
    \lhfmexampleclip{4}{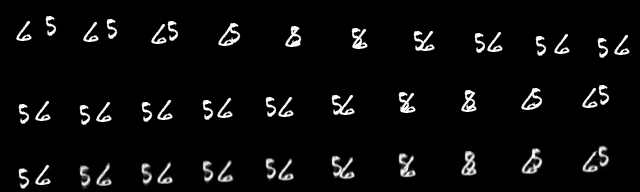}
    \vspace{7pt}
    \lhfmexampleclip{5}{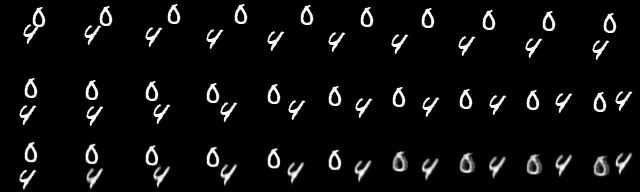}
    \caption{LHFM-V predictions on Moving MNIST (sequence IDs 3--5).
    Rows and time ordering follow Figure~\ref{fig:lhfm-v-examples-first}.
    Each observed sequence produces one deterministic prediction.}
    \label{fig:lhfm-v-examples-second}
\end{figure}
\clearpage
\endgroup

\end{document}